\documentclass{article}
\usepackage{amsmath}
\usepackage{amsthm}
\usepackage[english]{babel}
\usepackage[numbers,sort]{natbib}
\usepackage[letterpaper,top=2cm,bottom=2cm,left=3cm,right=3cm,marginparwidth=1.75cm]{geometry}

\usepackage{lmodern}
\usepackage{booktabs}
\usepackage{multicol}
\usepackage{multirow}
\usepackage[ruled,linesnumbered,vlined]{algorithm2e}
\usepackage{amsmath,bm}
\usepackage{graphicx}
\usepackage{authblk}
\usepackage[colorlinks=true, allcolors=blue]{hyperref}
\usepackage{cleveref}

\newtheorem{theorem}{Theorem}[section]

\newtheorem{lemma}{Lemma}[section]

\crefname{algorithm}{Algorithm}{Algorithm}
\crefname{lemma}{Lemma}{Lemma}
\crefname{table}{Table}{Table}
\crefname{theorem}{Theorem}{Theorem}
\crefname{corollary}{Corollary}{Corollary}
\crefname{equation}{Eq.}{Eq.}
\crefname{figure}{Fig.}{Fig.}
\crefname{section}{Section}{Section}
\crefname{appendix}{Appendix}{Appendix}

\title{LoPT: \textbf{Lo}ssless \textbf{P}arallel \textbf{T}okenization Acceleration for Long Context Inference of Large Language Model}
\author{Wei Shao, Lingchao Zheng, Pengyu Wang, Peizhen Zheng, Jun Li, Yuwei Fan\thanks{Corresponding author: fanyuwei2@huawei.com}}
\affil{Huawei}
\date{}
\begin{document}
\maketitle

\begin{abstract}
Long context inference scenarios have become increasingly important for large language models, yet they introduce significant computational latency. While prior research has optimized long-sequence inference through operators, model architectures, and system frameworks, tokenization remains an overlooked bottleneck. Existing parallel tokenization methods accelerate processing through text segmentation and multi-process tokenization, but they suffer from inconsistent results due to boundary artifacts that occur after merging. To address this, we propose LoPT, a novel Lossless Parallel Tokenization framework that ensures output identical to standard sequential tokenization. Our approach employs character-position-based matching and dynamic chunk length adjustment to align and merge tokenized segments accurately. Extensive experiments across diverse long-text datasets demonstrate that LoPT achieves significant speedup while guaranteeing lossless tokenization. We also provide theoretical proof of consistency and comprehensive analytical studies to validate the robustness of our method.
\end{abstract}

\section{Introduction}

With the advancing capabilities of large language models (LLMs) ~\cite{kimiteam2025kimik2openagentic,yang2025qwen3technicalreport,openai2025gptoss120bgptoss20bmodel}, long-context inference scenarios-such as document analysis~\cite{xu2025comprehensivesurveydeepresearch}, agent applications~\cite{luo2025largelanguagemodelagent}-have become increasingly important. However, processing long contexts introduces significant computational latency. Previous research has addressed this challenge by optimizing operators~\cite{dao2022flashattentionfastmemoryefficientexact}, model architectures~\cite{gu2024mambalineartimesequencemodeling,yang2025qwen3technicalreport,deepseekai2024deepseekv32}, and inference frameworks~\cite{lee2025architectinglongcontextllmacceleration}. Despite these efforts, tokenization time has emerged as a major contributor to inference latency as context lengths grow. Unfortunately, accelerating long-context tokenization remains underexplored.

Current approaches to tokenization acceleration fall into two categories, each with limitations. The first, \textbf{Algorithmic Acceleration}, improves the engineering implementation of tokenization algorithms~\cite{githubGitHubHuggingfacetokenizers,githubGitHubOpenaitiktoken}. While effective for texts of all lengths, this approach lacks specialized optimization for extremely long contexts (e.g., 64K tokens or more). The second strategy, \textbf{Chunk-based Parallel Tokenization}, splits the long text into shorter text chunks for parallel processing and merges the results. Depending on the segmentation method, it can be \textbf{Delimiter-based} or \textbf{Overlap-based}~\cite{githubGitHubOpenLMLabParallelTokenizer}. The former splits the long text based on the specific delimiter (such as whitespace, comma, period), and the latter includes overlapping regions between adjacent chunks to improve merging accuracy.

Although chunk-based methods achieve notable speedups, they suffer from a critical drawback: the merged result may not match the output of tokenizing the original text without splitting due to the change of text boundary (as shown in \Cref{example}). Such inconsistency of tokenization results may lead to degraded model task performance, which is a significant reason why this kind of method has not been widely adopted for long-context tokenization. Delimiter-based methods exhibit low accuracy, while overlap-based methods, though more accurate, still fall short of 100\% correctness. Moreover, the substantial computational overhead introduced by the overlap-based method's merging algorithms significantly undermines acceleration effectiveness. Currently, there is a lack of chunk-based parallel tokenization methods that can guarantee both perfectly accurate merging results and substantial speedup simultaneously.

\begin{figure*}
    \centering
    \includegraphics[width=0.9\linewidth]{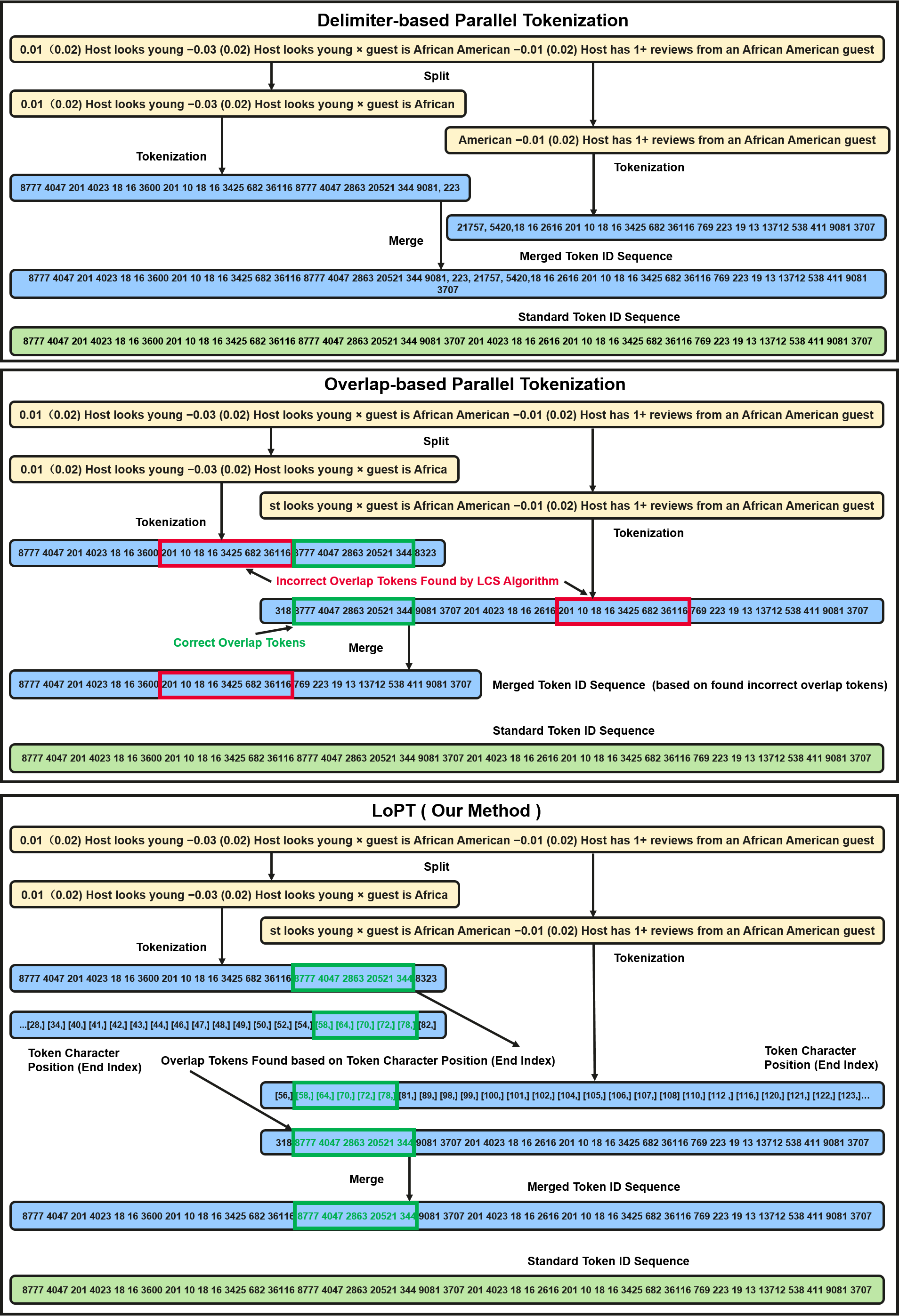}
    \caption{This figure shows the error cases of the previous chunk-based parallel tokenization methods (delimiter-based and overlap-based) and how our method LoPT obtains the correct result. Their outputs are in blue, and the standard tokenization results without splitting are in green. For delimiter-based parallel tokenization, incorrect results occur due to the neglect of token changes caused by variations in segment boundaries during the merge phase. For the overlap-based methods, although they attempted to address this issue by identifying overlap token sequences in the tokenization results of two adjacent text chunks, incorrect results still appear due to the mismatch of the overlap token sequence, which often occurs when processing certain texts, such as those containing consecutive repetitions. The LoPT identifies overlap tokens based on the token character position, which helps avoid such a mismatch.}
    \label{example}
\end{figure*}

To address these issues, we propose LoPT, a novel framework for \textbf{Lo}ssless \textbf{P}arallel \textbf{T}okenization and comparable tokenization speed with delimiter-based methods. Specifically, LoPT follows the Split $\rightarrow$ Parallel Tokenization $\rightarrow$  Merge pipeline from~\cite{githubGitHubOpenLMLabParallelTokenizer} but introduces key innovations. During splitting, the text is divided into overlapping chunks of equal length. Tokenization is performed in parallel, and results are merged based on character-level position information of tokens in the original text. This approach ensures accurate identification of tokens in overlapping regions without expensive token ID comparisons. Additionally, LoPT dynamically adjusts chunk lengths when the length of overlap sequences is insufficient, guaranteeing lossless results while maintaining speed comparable to delimiter-based methods. An example of how the LoPT processes a long text is also shown in \Cref{example}.

Experimental validation on three long text datasets (LongBenchV2~\cite{bai2025longbenchv2deeperunderstanding}, LEval (medium sub-datasets) ~\cite{an2023leval}, and ClongEval~\cite{qiu2024clongeval}) spanning various domains, formats, and languages demonstrates that LoPT significantly accelerates tokenization while ensuring perfect consistency with standard tokenization outputs. We also provide a theoretical proof of losslessness and conduct extensive analysis studies to offer further insights.

Based on our work, the contributions can be summarized as follows:

\begin{enumerate}
    \item We propose a novel lossless parallel tokenization framework (LoPT) that innovatively addresses the inconsistency issue in parallel tokenization through character position-based matching and dynamic chunk length adjustment.
    \item We conduct extensive experimental validation across multiple long-text datasets spanning diverse domains, formats, and languages, which demonstrates that our framework significantly accelerates tokenization while guaranteeing lossless results.
    \item We perform comprehensive theoretical analysis and a series of analytical experiments to prove the lossless nature of our method and investigate its performance under different settings, providing both theoretical foundation and practical insights.
\end{enumerate}

\section{Related Work}

\subsection{Tokenization Algorithm}

Tokenization is the process of converting text into input for large language models. In this paper, the term "tokenization" refers to the process of using a pre-constructed vocabulary to convert input text into token IDs for input to large language models. This process relies on a token vocabulary trained on large-scale data. To tokenize the texts more efficiently, several classic algorithms have been developed, including BPE~\cite{sennrich-etal-2016-neural}, WordPiece~\cite{Schuster2012JapaneseAK}, SentencePiece~\cite{kudo-richardson-2018-sentencepiece}, and BBPE~\cite{Wang_Cho_Gu_2020}. Among these, BPE and WordPiece are widely employed by many large language models as tokenization algorithms due to their efficiency and universality. We mainly discuss these two algorithms in our paper.

\subsection{Tokenization Acceleration}

Current tokenization primarily relies on well-established software libraries such as Hugging Face Tokenizers~\cite{githubGitHubHuggingfacetokenizers} and OpenAI TikToken~\cite{githubGitHubOpenaitiktoken}. Both libraries have implemented numerous engineering optimizations. For instance, Hugging Face Tokenizers offers a Rust-based fast version of the tokenizer, which significantly improves speed compared to the Python version. Similarly, TikToken has also implemented a Rust-based version, incorporating optimizations in storage and retrieval such as multi-threading, caching, and hashing.

The aforementioned tokenizers work on CPUs. There are also efforts aimed at accelerating the tokenization process from the GPU side. For example, cuDF~\cite{githubGitHubRapidsaicudf} has achieved acceleration for WordPiece tokenization on GPUs, and BlockBPE~\cite{you2025blockbpeparallelbpetokenization} has implemented a GPU-accelerated version of the BPE tokenization algorithm. Here, our focus is primarily on accelerating CPU-side tokenizers, as this remains the mainstream usage scenario and offers significant room for optimization.

Beyond engineering implementations that accelerate the tokenization process, some efforts have focused on improving tokenizer throughput through multi-process optimization. For example, the Hugging Face Datasets library supports multi-process processing for large-batch inputs, resulting in significant speedups. However, this type of optimization targets the overall processing of batch inputs and does not reduce the latency for individual samples. For inference scenarios, accelerating small-batch or even single-input processing is more critical, which is also the primary focus of our work.

\subsection{Tokenization Acceleration for Long-Context}

Although various acceleration techniques for tokenization have been mentioned earlier, some of these methods experience diminishing acceleration effects as sequence length increases. Other techniques are less suitable for long-context scenarios—for example, multi-process acceleration for batch data processing, as batch sizes of long-context inputs are typically small. These limitations result in excessively long tokenization times for sequences that are excessively long.

Currently, there are very few efforts specifically focused on accelerating tokenization for long contexts. Based on our research, only one software library, ParallelTokenizer~\cite{githubGitHubOpenLMLabParallelTokenizer}, attempts to address this issue from a multi-process perspective. Their approach involves splitting the long context into several shorter text chunks with overlapping segments between adjacent ones, processing these shorter texts using multiple processes, and finally merging the results. However, their merging strategy is relatively simple: adjacent segments are combined based on the longest overlapping token sequences. While this method achieves acceleration, it cannot guarantee that the merged result will be consistent with the output of direct tokenization on the original long context. Besides, due to the computation burden of finding an overlap token sequence, this method's acceleration performance is not significant. These are the problems our work aims to solve.

\begin{algorithm}[t]
	\caption{BPE Tokenization Algorithm}
	\label{BPE Tokenization}
	\KwIn{Input sentence string: $S$; Vocabulary: $V$; Merge Table: $M$.}
	\KwOut{Token ids list of input sentence $T$.}  
	\BlankLine
	Normalize and pre-tokenize $S$ to get a list of words $W$ ;

        Initilize $T=[]$
        
	\ForEach{i=0  to len($W$)-1}{
         Split $W[i]$ into a list of chars $C$;

        \While{True}{
         (best pair, best score) = (null, 0);
         
         \ForEach{j=0 to len(C)-1}{
            score = GET($M$, C[j], C[j+1]);

           \If{$\text{score}  > \text{best score}$}{
            (best pair, best score) = ((C[j], C[j+1]), score);
           }
         }
         \If{best pair is null}{
         break // can no longer be merged
         }
         \ForEach{j=0 to len(C)-1}{
           \If{$(C[j], C[j+1])  = \text{best pair}$}{
            MERGE(C[j], C[j+1]);
           }
         }
        }
        Replace tokens in $C$ with token id in $V$;
        
        Append $C$ to $T$;
        }
\end{algorithm}

\begin{algorithm}[t]
	\caption{WordPiece Tokenization Algorithm}
	\label{WordPiece Tokenization}
	\KwIn{Input sentence string: $S$; Vocabulary: $V$;}
	\KwOut{Token ids list of input sentence $T$.}  
	\BlankLine
	Normalize and pre-tokenize $S$ to get a list of words $W$ ;

        Initilize $T=[]$
        
	\ForEach{i=0  to len($W$)-1}{
    
        \While{$W[i]$}{
            $v_i$ = MatchLongestSubwordFromStart($W[i]$, $V$);
        
            \eIf{$v_i$}
            {
                Append $v_i$'s token id to the $T$;
                
                $W[i] = W[i] - v_i$;
            }{
                Append unknown token's id to the $T$;
            }
        }
    }
\end{algorithm}

\begin{figure}
    \centering
    \includegraphics[width=0.9\linewidth]{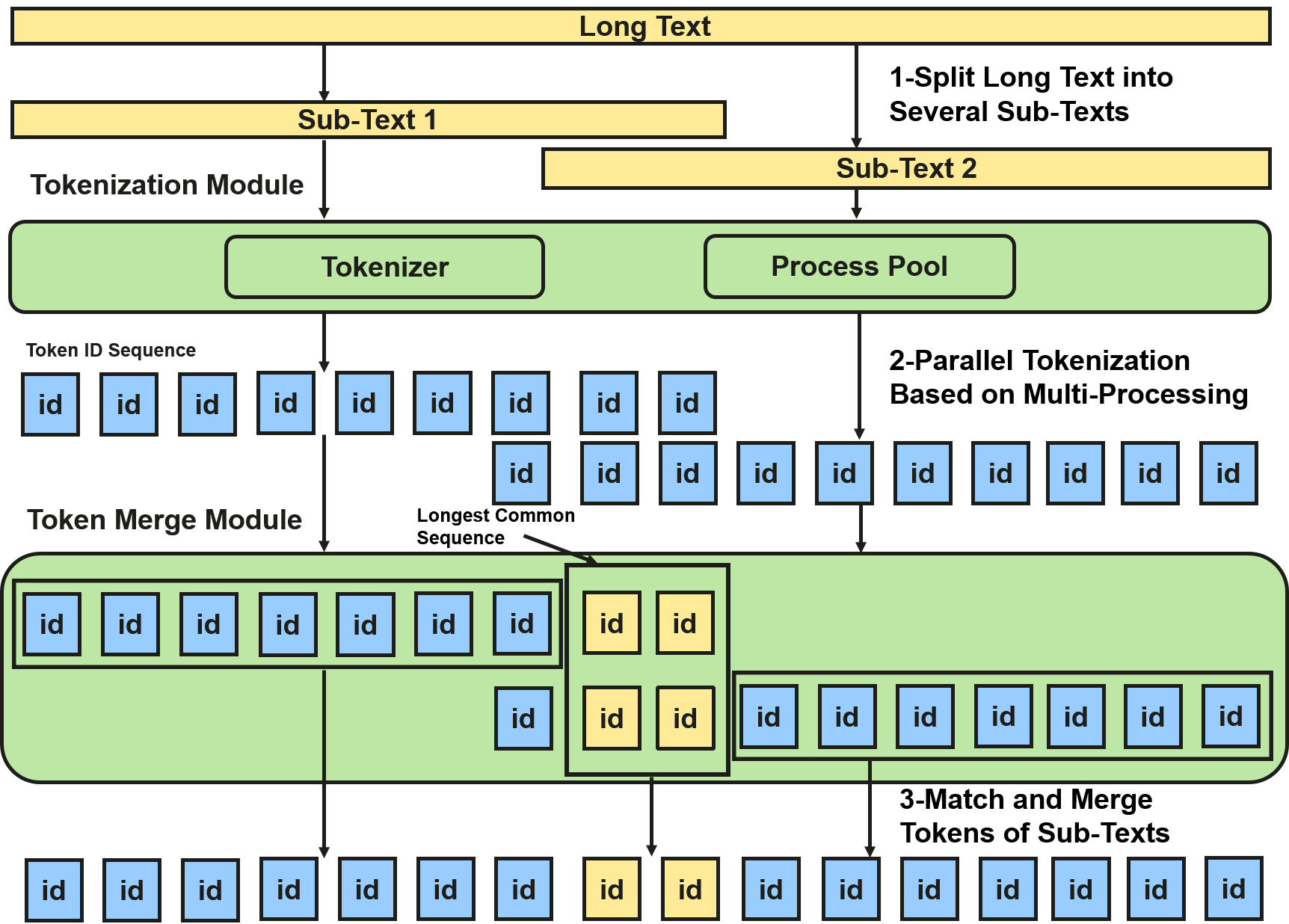}
    \caption{The ParallelTokenizer's process to parallelly tokenize a long text. Here, we use two text segments as an example.}
    \label{ParallelTokenizer}
\end{figure}

\section{Preliminary}

\subsection{WordPiece and BPE Tokenization}



We summarize the process of WordPiece and BPE tokenization algorithms in \Cref{BPE Tokenization} and \Cref{WordPiece Tokenization}. 

For BPE tokenization, assume that we have obtained
vocabulary $V$ and merge table $M$ from BPE training, and we need to convert the input sentence $S$ into a list of token IDs. The first step is to normalize and pre-tokenize the input sentence $S$. The detailed process depends on the specific tokenizer. After this step, we can obtain a list of words $W$. Then, for each word in $W$, we will split it into a basic character list $C$. For each adjacent pair in $C$, we will look up the merge table $M$ to get its score and then determine and merge the best adjacent pair. This process will be repeated until no pairs can be merged. After obtaining a word's BPE tokens, we will map these tokens to their vocabulary IDs and append them to the final results $T$.

For WordPiece tokenization, the input sentence string $S$ is split into a word list $W$, and each word is processed individually. For each word, the algorithm attempts to match the longest sub-word from the vocabulary $V$. Matched sub-words $v_i$ are added to the output, while unmatched parts are replaced with an unknown token. This process repeats for each word in the input.

\subsection{Chunk-based Parallel Acceleration for Long-Context Tokenization}

Chunk-based parallel tokenization acceleration methods can be divided into two categories. The first category involves splitting long texts based on a specific delimiter (such as periods or commas), then applying multi-process tokenization, and finally directly concatenating the tokenization results. However, this approach often alters the tokenization outcomes at the boundaries of the split text segments, leading to inconsistencies with the standard tokenization results. To alleviate this problem, the second method employs an overlapping segmentation method, where adjacent text segments are split with a certain degree of overlap. The tokenization results of these adjacent segments are then merged based on overlapping tokens. A representative method of this approach is the ParallelTokenizer~\cite{githubGitHubOpenLMLabParallelTokenizer}.

Here, we describe how the ParallelTokenizer speeds up the tokenization of long contexts. According to \Cref{ParallelTokenizer}, given the input long sequence text $S$, the ParallelTokenizer will split it into $N$ shorter texts. A tokenization process will process each text chunk, and these texts are tokenized nearly simultaneously. Then, the token IDs of each adjacent text pair are merged. In this figure, we use two text chunks as an example to explain the merge process. The sub-text 1's tokens have a certain length of overlap with the sub-text 2's tokens. The ParallelTokenizer finds the overlap tokens (in yellow) with the longest common sequence algorithm based on token IDs. Then, the left tokens in sub-text 1 (in blue) and the overlap tokens are merged with the right tokens in sub-text 2 (in blue) to form the final tokens.

As mentioned earlier, although this method can achieve acceleration, it cannot guarantee that the final result will be consistent with that of standard tokenization. This is because relying solely on token ID matching does not ensure that the matched tokens correspond to the same positions in the original long context, which may result in missing or extra tokens in the merged output.

Moreover, we have observed that even if the tokens obtained from matching adjacent segments correspond to the same tokens in the original long text, there is still no guarantee that the final merged result will align with the standard tokenization output. This discrepancy arises because segmentation alters the context around overlapping tokens, potentially leading to the application of different rules during tokenization. As a result, the token IDs may differ from those generated by tokenizing the original long text directly.

\section{Methodology}

\subsection{Overview}
To achieve the harmony between lossless results and acceleration, we propose a novel lossless parallel tokenization framework (LoPT) for long context tokenization acceleration. This framework takes long text as input and outputs its corresponding tokens. As shown in \Cref{architecture}, the framework consists of three modules: the text split module, the parallel tokenization module, and the position-aware token merge module. The functions of each module are as follows:

\textbf{Text Split Module}: Receives the long text input and splits it into several overlapped text chunks based on predefined chunk length and overlap length.

\textbf{Parallel Tokenization Module}: Utilizes multi-processing to tokenize each text chunk in parallel.

\textbf{Position-Aware Token Merge Module}: Merging two adjacent text chunks' tokenization results based on their overlap token sequence. For each token in the overlap token sequence, it is essential to ensure that its character position in the original long text, calculated based on its character positions in the two adjacent text chunks, remains consistent.

Besides, we also introduce a dynamic chunk length mechanism to ensure the existence of an overlap token sequence between two adjacent text chunks. In detail, for each merging of two adjacent text chunks, it dynamically adjusts the chunk length based on a minimum token match threshold.

These three modules and the mechanism work together to ensure faster tokenization and a lossless result. In the following, we will introduce them in detail.

\subsection{Lossless Parallel Tokenization}

\begin{figure}
    \centering
    \includegraphics[width=1.0\linewidth]{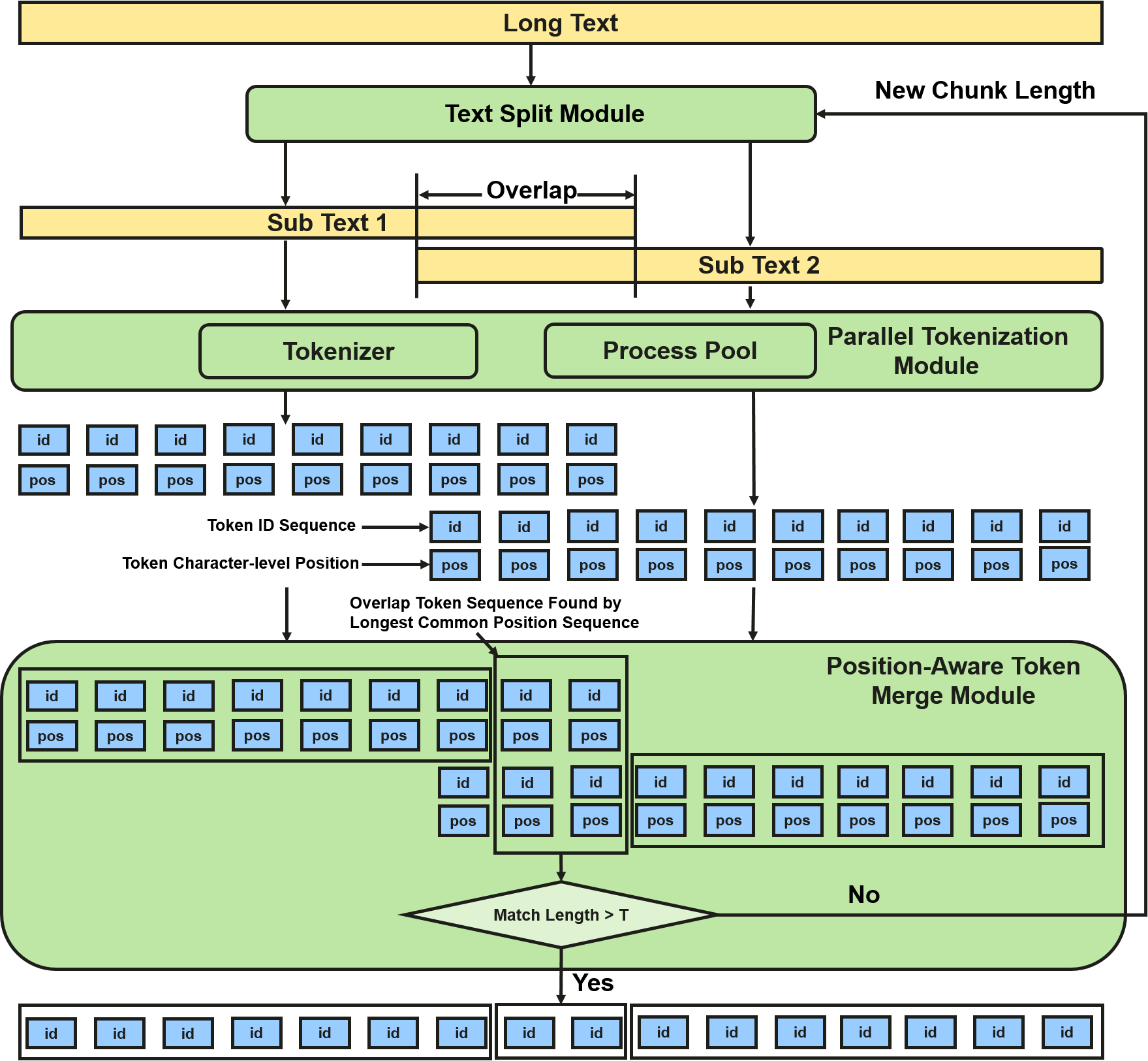}
    \caption{This figure illustrates the workflow of our method, with the green sections representing our method modules. The input long text  is first split by the text split module into text chunks of chunk length, with adjacent text chunks having an overlap of overlap length. Then, each text chunk is processed by the parallel tokenization module to generate the corresponding token sequence. Unlike previous methods, at this step, we also output the character-level position of each token within the corresponding text chunk, providing the necessary positional information for merging the tokenization results. Finally, the token merge module receives the tokens and token position information of each text chunk and performs the merging operation.}
    \label{architecture}
\end{figure}

\subsubsection{Text Split Module}

To ensure the losslessness of the parallel method, a certain overlap region is required between adjacent text chunks. Therefore, we adopt a process similar to the ParallelTokenizer, dividing the long text $S$ into sub-segments $\{s_1, s_2, ..., s_N\}$ based on a fixed chunk size $L_c$ and overlap length $L_o$. This process could be represented as:

\begin{equation}
    s_i = S[L_c*(i-1):L_c*i + L_o].
\end{equation}


\subsubsection{Parallel Tokenization Module}

This module serves as the primary component for acceleration. It speeds up the tokenization process by leveraging multi-process parallelization to invoke the tokenizer for processing. Moreover, due to the need for subsequent matching based on token character position information, in this step, we require the tokenizer to output the tokens along with their corresponding character positions within each text chunk. The whole process could be represented as:

\begin{equation}
     (T_i, P_i) = Process_i(Tokenizer(s_i)),
\end{equation}
where $s_i$ is the i-th text chunk, $T_i$ is the token id list of $s_i$ and $P_i$ is the character-level position of each token in $s_i$. 

\subsubsection{Position-Aware Token Merge Module}

This module ensures lossless final results and consists of two steps: match and merge. The match step aims to find the longest overlap token sequence between adjacent text chunk tokenization results based on the token's character positions. The merge step aims to merge these tokenization results based on the found overlap token sequences. 

Specifically, during the match step, we incorporate token-to-character position correspondence to ensure that matched overlapping tokens fall within the overlapping character range of both text chunks. The identified token sequences must align exactly at their starting and ending character boundaries to prevent missing or extra tokens caused by misalignment. Moreover, since position information is inherently ordered, we can leverage algorithms with lower computational complexity to identify overlapping token sequences, thereby reducing the time cost of this component compared to the previous method. During the merge step, given the overlapping region [a, b] between text chunks A and B, we retain the tokens from text chunk A before position a identified overlap token sequence, and the tokens from text chunk B after position b, then concatenate them. This process is repeated iteratively for the remaining text chunks. 

According to our experiment results, a sufficient length of the overlap token sequence is necessary for lossless tokenization results. However, if the preset chunk length is not suitable, we may not find such an overlap token sequence. To this end, we evaluate whether the current chunk length is appropriate based on the length of the obtained overlap token sequence. If it is smaller than a threshold, the text chunk length will be increased, and the whole process will restart based on the new chunk length.  

Formally, the match process can be represented as:

\begin{equation}
     l_{i}, r_{i}, n_{i}^o = Match(P_i, P_{i+1})
\end{equation}
, where $l_{i}$ and $r_{i}$ are indexes of the start token of overlap tokens in the $T_i$ and $T_{i+1}$. $n_{i}^o$ is the number of valid overlap tokens. \textbf{Note that the valid overlap tokens' char-level positions in $s_i$ and $s_{i+1}$ must have a fixed offset: chunk length. This means that the global char-level position (in the long text) of the token is computed based on the char-level position in $s_i$ is the same as that computed based on the char-level position in $s_{i+1}$}.

In the match process, if two adjacent text chunks have no sufficient overlap tokens with the same global char-level position in the long text ($n_{i}^o = 0$), this match will be regarded as a failed match, and the current chunk size will be doubled as the input parameter of a new long text split process. If the match is successful, we will merge token lists ${T_1, T_2, ..., T_N}$ as the final token list $T$ based on the following process:

\begin{equation}
    T = concatenate(T_1[:l_1 + n_{1}^{o}], T_2[r_1+n_{1}^o:l_2 + n_{2}^{o}],... , T_N[r_{N-1}+n_{N-1}^o:])
\end{equation}

Here, we dynamically adjust the chunk length based on the actual matching results to ensure the presence of overlap tokens at identical positions, thereby guaranteeing lossless final merge results.

\subsubsection{The LoPT Algorithm}

In this subsection, to more clearly demonstrate the process of our proposed method, we summarize the processing flow of the LoPT into~\Cref{lopt_alg}.

\begin{algorithm}[t]
\caption{The LoPT Algorithm}
\label{lopt_alg}
\KwIn{Long Text: $S$; A tokenization tool $Tokenizer$; The length of text chunk $L_c$, The length of overlap between two adjacent text chunks $L_o$, A process pool with size $M$}
\KwOut{Final token ID list of input sentence $T$.}  
\BlankLine
    Initialize $T=[]$;
    
    Initialize Overlap $O=[]$;

    Split $S$ into $N$ text chunks $s_i = S[L_c*(i-1):L_c*i + L_o], i=1, 2, .., N$;

    Assign process ($Process_{i}$) from the process pool for each $s_i$'s tokenization;
    
    $(T_i, P_i) = Process_i(Tokenizer(s_i))$; // {parallelly tokenize each text chunk $s_i$ to obtain token ID list and token positions}
    
    \ForEach{ i=1  to N-1 }{
        $ l_{i}, r_{i}, n_{i}^o = Match(P_i, P_{i+1}) $;  // {get start tokens' indexes of overlap in $T_i$ and $T_{i+1}$ and the number of overlapped tokens}

        \eIf{$ n_i^o > 0 $}
        {
            Append  ($l_{i}$, $r_{i}$, $n_{i}^o$) to $O$;
        }{  
            $L_c$ = $L_c$ * 2;
            
             Re-split $S$ with $L_c$ and $L_o$;
        }

    }

    $T = concatenate(T_1[:l_1 + n_{1}^{o}], T_2[r_1+n_{1}^o:l_2 + n_{2}^{o}], ... , T_N[r_{N-1}+n_{N-1}^o:])$;
    
\end{algorithm}

\subsection{Theory Analysis}
In the section background, we provide the algorithm of the WordPiece and BPE tokenization algorithms. Here, based on them, we try to analyze and demonstrate that when specific conditions are met, the overlap-based parallel tokenization method can ensure consistency between the merged results and the original tokenization results.

So what is this condition? We believe that, for the previous overlap-based parallel tokenizer method, the overlap tokens identified are not truly common. For tokens to be considered overlapping, they must satisfy the condition that their span start positions in the original long text are consistent, which is ensured by our framework. In \Cref{app:proof}, we attempt to prove that merging based on such overlap tokens can ensure the final result of overlap-based parallel tokenization is consistent with the standard results of the two types of tokenization algorithms, as stated in \Cref{thm:lossess}.

\begin{theorem} \label{thm:lossess}
Denote $s_i = S[L_c*(i-1):L_c*i+L_o]$, $(T_i, P_i) = Tokenizer(s_i)$, 
$(l_i, r_i, n_i^o) = Match(P_i, P_{i+1})$, and $T = concatenate(T_1[:l_1 + n_{1}^{o}], T_2[r_1+n_{1}^o:l_2 + n_{2}^{o}], ... , T_N[r_{N-1}+n_{N-1}^o:])$.

If $n_i^o > 0, 1\leq i\leq N-1$, then 
$$T = Tokenizer(S).$$
\end{theorem}
\section{Experiment}

To validate the effectiveness of our proposed framework, we conduct a series of comparative experiments. Furthermore, to investigate the performance of our method across various scenarios, we also conduct a set of analytical experiments. In the following section, we will introduce these two major categories of experiments.

\subsection{Comparison Experiment}
In this section, we aim to answer the following questions through a series of experiments: How does the acceleration performance of our method compare with that of the original tokenizer in long-context scenarios? Can our method achieve lossless tokenization with comparable processing time relative to other parallel acceleration approaches? To answer these questions, we conduct comparison experiments using different tokenizers on three long-text datasets. Details of the comparison experiment settings are presented below.

\subsubsection{Experiment Settings}

\paragraph{Datasets}

 We use three long-text datasets: LongBenchV2~\cite{bai2025longbenchv2deeperunderstanding}, LEval (medium sub-datasets) ~\cite{an2023leval}, and ClongEval~\cite{qiu2024clongeval} as benchmark datasets. The LongBenchV2 and LEval consist of English long texts, and the ClongEval consists of Chinese long texts. The average text length of LongBenchV2 is significantly larger than that of the other two datasets.

\paragraph{Implementation of LoPT}

We employ the HuggingFace fast tokenizer as the foundational tokenizer in our framework. To achieve optimal speed, the token merge module is implemented in C++, while the remaining components are developed in Python. Our framework defaults to using 32 processes for acceleration, with the chunk length set to the input text length divided by the number of processes. To obtain more accurate performance measurements for long-text processing, we set the input batch size to 1. The threshold for the minimum length of the overlap token sequence is 2.

\paragraph{Baselines}
\begin{enumerate}
    \item HuggingFace Tokenizer Fast: It is the fast version of the tokenizer provided by HuggingFace and used to provide standard tokenization results and time baseline here. To ensure the robust and generalization of experiments, we use tokenizers from different model families, including Qwen3~\cite{yang2025qwen3technicalreport}, DeepSeek-V3~\cite{deepseekai2024deepseekv3technicalreport}, Llama-3.1~\cite{patterson2022carbonfootprintmachinelearning}, GPT-OSS-120B~\cite{openai2025gptoss120bgptoss20bmodel}, BERT-Base-Uncased, and BERT-Base-Cased~\cite{DBLP:journals/corr/abs-1810-04805}. Among them, the tokenizers of the first four models use the BPE tokenization algorithm, while the tokenizers of the latter two models use the WordPiece tokenization algorithm.

    \item Delimiter-based ParallelTokenizer: The long text is segmented according to a specific text delimiter, and then these text chunks are tokenized in parallel based on a multi-process approach. Finally, the tokenization results of adjacent text chunks are directly concatenated as the final result. Here, we employ three types of delimiters: whitespace, comma, and period. To obtain a fair comparison, we choose the right first delimiter around the chunk length.

    \item Overlap-based ParallelTokenizer: The long text is split into overlapping adjacent chunks, which are then tokenized in parallel using a multi-process approach. The results of adjacent segments are merged based on overlapping token IDs. Here, we use the ParallelTokenizer~\cite{githubGitHubOpenLMLabParallelTokenizer} as the representative method. The original version was implemented in Python, and the algorithm used to identify the overlap token sequence could be improved. To ensure a fairer comparison, we reimplement the framework and optimize its token merge algorithm for more accurate results. We will list the results of the two versions in the following comparison experiment. (ParallelTokenizer-Origin, ParallelTokenizer-Ours)

\end{enumerate}

\paragraph{Metrics}

We measured the end-to-end \textbf{latency} of the standard HuggingFace TokenizerFast, our framework, and other baseline methods, as well as the \textbf{accuracy} of our framework and other parallel tokenization methods. Specifically, end-to-end latency means the time a tokenizer converts a sentence into a dictionary containing tensors like input IDs, attention masks. The accuracy refers to the proportion of results that exactly match those obtained by directly using the standard HuggingFace TokenizerFast.

\paragraph{Device}

Unless otherwise specified, all experiments were conducted on a CPU environment with a clock speed of 3.8 GHz and 112 cores.

\subsubsection{Comparison Results}

The experimental results are shown in \Cref{comparison_result}. Compared to other methods, our approach is the only acceleration method that achieves an accuracy of 1 on all datasets with all tokenizers, which demonstrates the effectiveness and robustness of our framework. The Overlap-based ParallelTokenizer also achieves high accuracy on the three datasets, but its acceleration performance is significantly weaker than our method due to the higher complexity of its merge algorithm. While the remaining three delimiter-based methods achieve significant acceleration, their accuracy is relatively low and is influenced by the tokenizer's vocabulary and the language of the input context, which limits their generalization.

Moreover, we can observe that BERT-like tokenizers achieve significantly higher accuracy when using the delimiter-based method compared to other tokenizers. This is because during preprocessing, BERT-like tokenizers typically treat punctuation and spaces as separate tokens, making the segmentation and merging results consistent for most texts. In contrast, other tokenizers tend to generate sub-word tokens where punctuation and spaces may be embedded within the tokens, leading to corruption of the token sequence after segmentation.

\begin{table}[]
    \centering
    \resizebox{\textwidth}{!}{
    \begin{tabular}{cccccccc}
    \toprule
    \multirow{2}{*}{LLM} & \multirow{2}{*}{Tokenization Methods} & \multicolumn{2}{c}{LongBenchV2} & \multicolumn{2}{c}{LEval} &  \multicolumn{2}{c}{ClongEval} \\
    &  & Latency(ms) & Accuracy & Latency(ms) & Accuracy & Latency(ms) & Accuracy \\
    \midrule
    \multirow{6}{*}{Qwen3} & 
    HuggingFace TokenizerFast & 618.5 & - & 33.2 & - & 69.9 & - \\
     & Delimiter based ParallelTokenizer (,) & 110.9 & 0.247 & 5.7 & 0.315 & 21.1 & 0.576 \\
     & Delimiter based ParallelTokenizer (.) & 103.7 & 0.328 & 6.4 & 0.559 & 14.4 & 0.677 \\
     & Delimiter based ParallelTokenizer (-) & 93.0 & 0.708 & 5.2 & 0.806 & 15.3 & 0.711 \\
     & Overlap-based ParallelTokenizer-Origin & 644.7 & 0.879 & 32.3 & 0.803 & 68.2 & 0.946 \\
     & Overlap-based ParallelTokenizer-Ours & 332.4 & 0.982 & 8.5 & 0.998 & 18.6 & 0.999 \\
     & LoPT(Ours) & 116.8 & 1.0 & 7.6 & 1.0 & 16.8 & 1.0 \\
    \midrule
    \multirow{6}{*}{DeepSeek-V3} & 
    HuggingFace TokenizerFast & 622.4 & - & 32.4 & - & 74.0 & - \\
     & Delimiter based ParallelTokenizer (,) & 104.5 & 0.247 & 5.2 & 0.317 & 20.0 & 0.575  \\
     & Delimiter based ParallelTokenizer (.) & 98.6 & 0.326 & 6.3 & 0.564 & 14.1 & 0.605 \\
     & Delimiter based ParallelTokenizer (-) & 91.2 & 0.700 & 5.1 & 0.804 & 16.0 & 0.666 \\
     & Overlap-based ParallelTokenizer-Origin & 648.2 & 0.889 & 40.6 & 0.801 & 80.5 & 0.946 \\
     & Overlap-based ParallelTokenizer-Ours & 224.3 & 0.98 & 8.4 & 0.998 & 17.8 & 0.999 \\
     & LoPT(Ours) & 108.2 & 1.0 & 7.2 & 1.0 & 15.6 & 1.0 \\
    \midrule
    \multirow{6}{*}{Llama-3.1} & 
    HuggingFace TokenizerFast & 513.3 & - & 26.1 & - & 60.4 & - \\
     & Delimiter based ParallelTokenizer (,) & 91.2 & 0.247 & 6.6 & 0.315 & 18.2 & 0.573  \\
     & Delimiter based ParallelTokenizer (.) & 87.1 & 0.328 & 7.8 & 0.559 & 13.8 & 0.665 \\
     & Delimiter based ParallelTokenizer (-) & 81.1 & 0.708 & 7.4 & 0.806 & 15.8 & 0.711 \\
     & Overlap-based ParallelTokenizer-Origin & 488.9 & 0.887 & 26.1 & 0.803 & 60.6 & 0.944 \\
     & Overlap-based ParallelTokenizer-Ours & 218.7 & 0.980 & 7.9 & 0.998 & 19.4 & 0.998 \\
     & LoPT(Ours) & 103.8 & 1.0 & 7.0 & 1.0 & 17.0 & 1.0 \\
    \midrule
    \multirow{6}{*}{GPT-OSS-120B} & 
    HuggingFace TokenizerFast & 500.6 & - & 26.9 & - & 66.3 & - \\
     & Delimiter based ParallelTokenizer (,) & 98.3 & 0.245 & 5.6 & 0.315 & 19.9 & 0.572  \\
     & Delimiter based ParallelTokenizer (.) & 92.0 & 0.324 & 5.7 & 0.561 & 14.5 & 0.610 \\
     & Delimiter based ParallelTokenizer (-) & 87.4 & 0.708 & 6.1 & 0.806 & 15.9 & 0.711 \\
      & Overlap-based ParallelTokenizer-Origin & 505.0 & 0.887 & 26.5 & 0.803 & 64.1 & 0.944 \\
     & Overlap-based ParallelTokenizer-Ours & 218.5 & 0.982 & 7.8 & 0.998 & 18.9 & 0.998 \\
     & LoPT(Ours) & 108.7 & 1.0 & 6.4 & 1.0 & 16.6 & 1.0 \\
    \midrule
    \multirow{6}{*}{BERT-BASE-CASED} & 
    HuggingFace TokenizerFast & 509.2 & - & 23.6 & - & 63.5 & - \\
     & Delimiter based ParallelTokenizer (,) & 93.7 & 0.970 & 6.7 & 0.966 & 18.1 & 0.834 \\
     & Delimiter based ParallelTokenizer (.) & 91.4 & 0.956 & 7.0 & 0.953 & 14.0 & 0.868 \\
     & Delimiter based ParallelTokenizer (-) & 86.6 & 0.996 & 6.5 & 0.994 & 16.6 & 0.715 \\
     & Overlap-based ParallelTokenizer-Origin & 526.4 & 0.865 & 22.9 & 0.797 & 64.1 & 0.943 \\
     & Overlap-based ParallelTokenizer-Ours & 479.5 & 0.982 & 7.6 & 0.996 & 29.0 & 0.997  \\
     & LoPT(Ours) & 106.7 & 1.0 & 6.3 & 1.0 & 17.3 & 1.0 \\
    \midrule
    \multirow{6}{*}{BERT-BASE-UNCASED} & 
    HuggingFace TokenizerFast & 578.0 & - & 26.5 & - & 73.6 & - \\
     & Delimiter based ParallelTokenizer (,) & 94.9 & 0.970 & 6.7 & 0.966 & 19.3 & 0.834 \\
     & Delimiter based ParallelTokenizer (.) & 89.6 & 0.956 & 7.2 & 0.953 & 16.3 & 0.868 \\
     & Delimiter based ParallelTokenizer (-) & 83.8 & 0.996 & 6.6 & 0.994 & 17.9 & 0.715 \\
     & Overlap-based ParallelTokenizer-Origin & 571.5 & 0.867 & 31.3 & 0.793 & 75.2 & 0.945 \\
     & Overlap-based ParallelTokenizer-Ours & 431.7 & 0.980 & 7.5 & 0.996 & 26.1 & 0.997 \\
     & LoPT(Ours) & 107.1 & 1.0 & 6.6 & 1.0 & 18.5 & 1.0 \\
    \bottomrule
    \end{tabular}}
    \caption{The performance of different tokenization methods on three datasets. We use standard HuggingFace TokenizerFast from Qwen3, DeepSeek-V3, Llama-3.1, GPT-OSS-120B, BERT-Base-Cased, and BERT-Base-Uncased as the base tokenizers. The latency is the time (in milliseconds) required to convert a string to a dictionary consisting of a series of tensors. The accuracy is the proportion of results that exactly match those obtained by directly using the HuggingFace TokenizerFast. For the delimiter-based parallel tokenizer, ("-"), (","), (".") indicate that we use whitespace, comma, and period as the delimiters to split the long text.}
    \label{comparison_result}
\end{table}

\subsection{Analysis Experiment}
\subsubsection{Sequence Length's Impact}

\begin{figure*}
    \centering
    \includegraphics[width=0.9\linewidth]{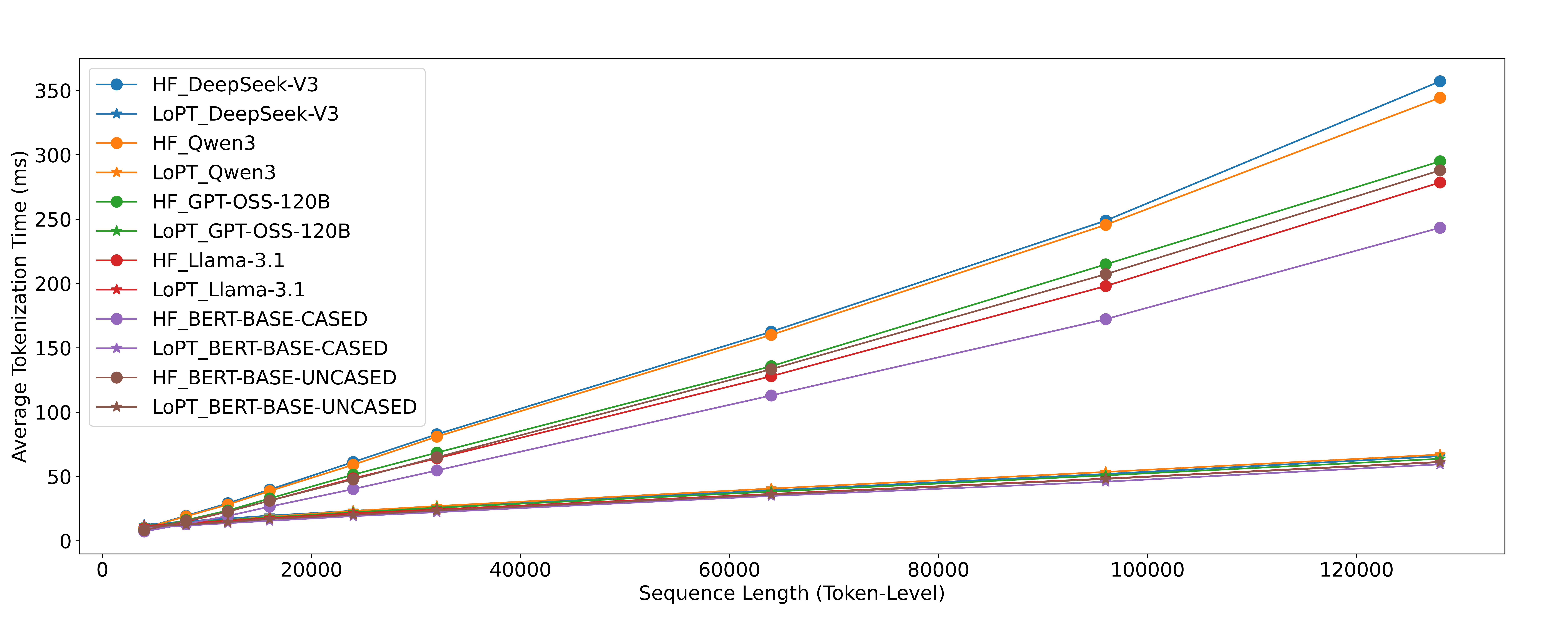}
    \caption{A performance comparison between the HuggingFace TokenizerFast ("·" points) and our framework LoPT ("*" points) on the LongBenchV2 dataset, with different sequence lengths. The horizontal axis represents the sequence length (number of tokens), and the vertical axis represents the tokenization time consumption (ms).}
    \label{seq_len}
\end{figure*}

To investigate the performance differences between our method and the standard tokenization approach under varying sequence lengths (number of tokens), we conducted experiments at sequence lengths of 4k, 8k, 24k, 32k, 40k, 48k, 64k, and 128k using the LongBenchV2 dataset. The performance of both the standard tokenizer and our framework is shown in \Cref{seq_len}.

According to the results, we can observe that the processing time of the HuggingFace TokenizerFast increases linearly with the sequence length. Although our framework also exhibits linear growth, its slope is significantly smaller than that of the HuggingFace TokenizerFast. This indicates that the advantage of our method becomes more pronounced as the sequence length increases.

\subsubsection{Chunk Size's Impact}

\begin{figure*}
    \centering
    \includegraphics[width=0.9\linewidth]{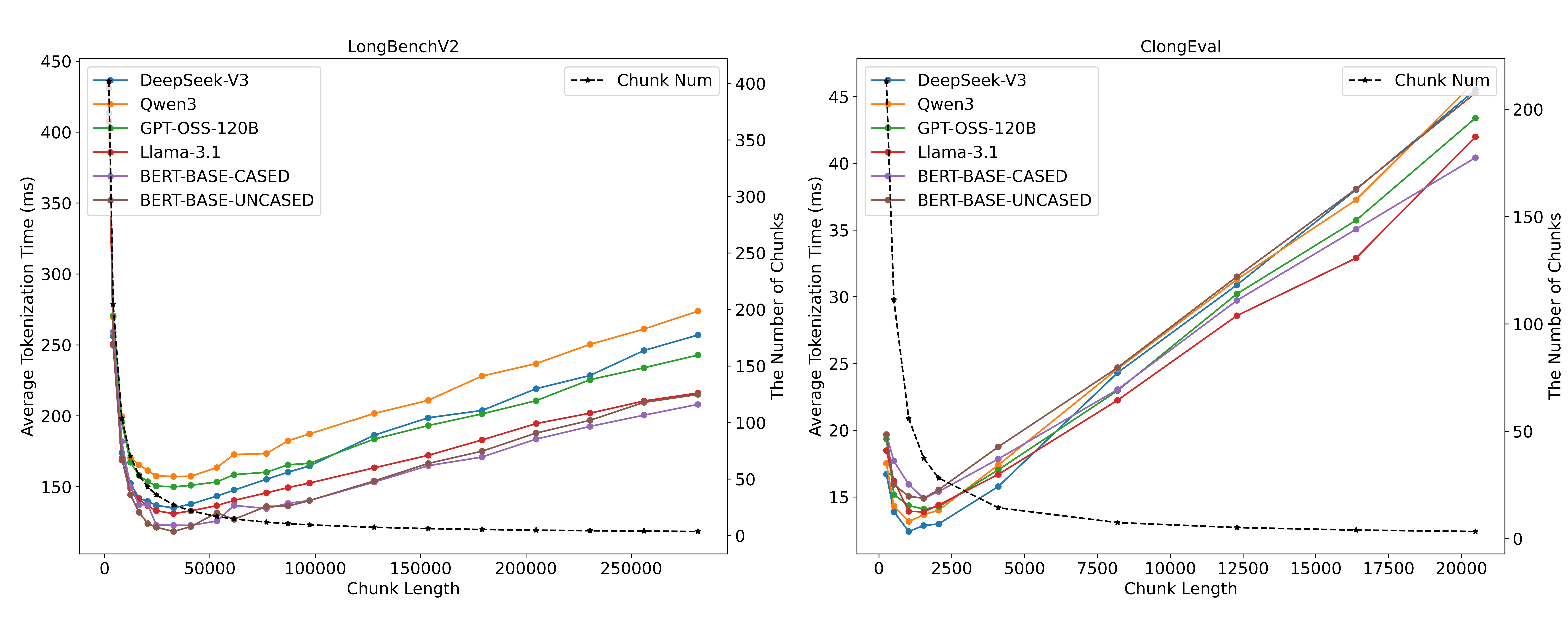}
    \caption{Performance of tokenization using our framework on LongBenchV2 and ClongEval datasets with different chunk lengths. The process pool size is 32. The horizontal axis represents the chunk length (character-level), and the vertical axis represents the tokenization time consumption. The dashed line represents the average number of chunks corresponding to different chunk lengths.}
    \label{chunk_size}
\end{figure*}

To investigate the performance of our framework with different chunk lengths (character-level), we conduct experiments to evaluate the performance of the LoPT across varying chunk lengths. The results are shown in \Cref{chunk_size}.

According to this figure, in the initial stage, due to the relatively small chunk lengths, the number of chunks generated is large, and the process pool is insufficient. As the chunk length increases, the number of chunks decreases. Although the processing time per process becomes longer, the waiting time for fragmented processes is shortened, leading to a reduction in overall time. However, as the chunk length continues to grow and the number of chunks further decreases, idle processes emerge in the process pool, and the processing time per process increases even more. During this stage, as the chunk size increases, the overall time required also increases.

Additionally, it can be observed that the type of tokenizer relatively less influences the chunk length corresponding to the minimum tokenization time. Theoretically, this turning point is more significantly affected by the input length and the pool size. Longer prompts and larger pool sizes lead to a larger chunk size at this inflection point. Overall, the chunk length at this point falls within a range that brings the number of chunks close to the size of the pool.

\subsubsection{Batch Size's Impact}
\begin{figure*}
    \centering
    \includegraphics[width=0.9\linewidth]{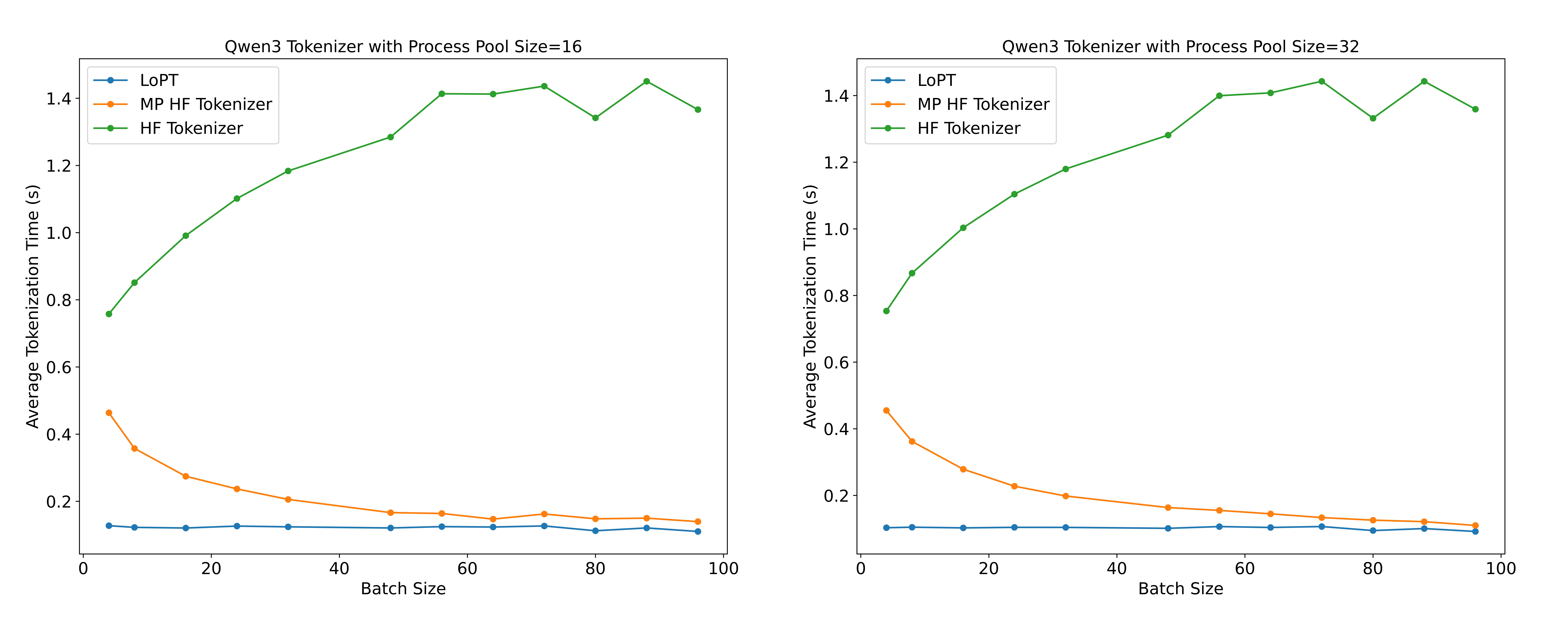}
    \caption{Performance of our framework and two baselines (we choose Qwen3 tokenizer as the base tokenizer) on the LongBenchV2 dataset with different batch sizes. The horizontal axis represents the batch size, and the vertical axis represents the time consumption for tokenization (s). HF Tokenizer and MP HF Tokenizer refer to the HuggingFace tokenizer and its sample-level, multiple-process version, respectively.}
    \label{batch_size}
\end{figure*}

\begin{figure}
    \centering
    \includegraphics[width=0.9\linewidth]{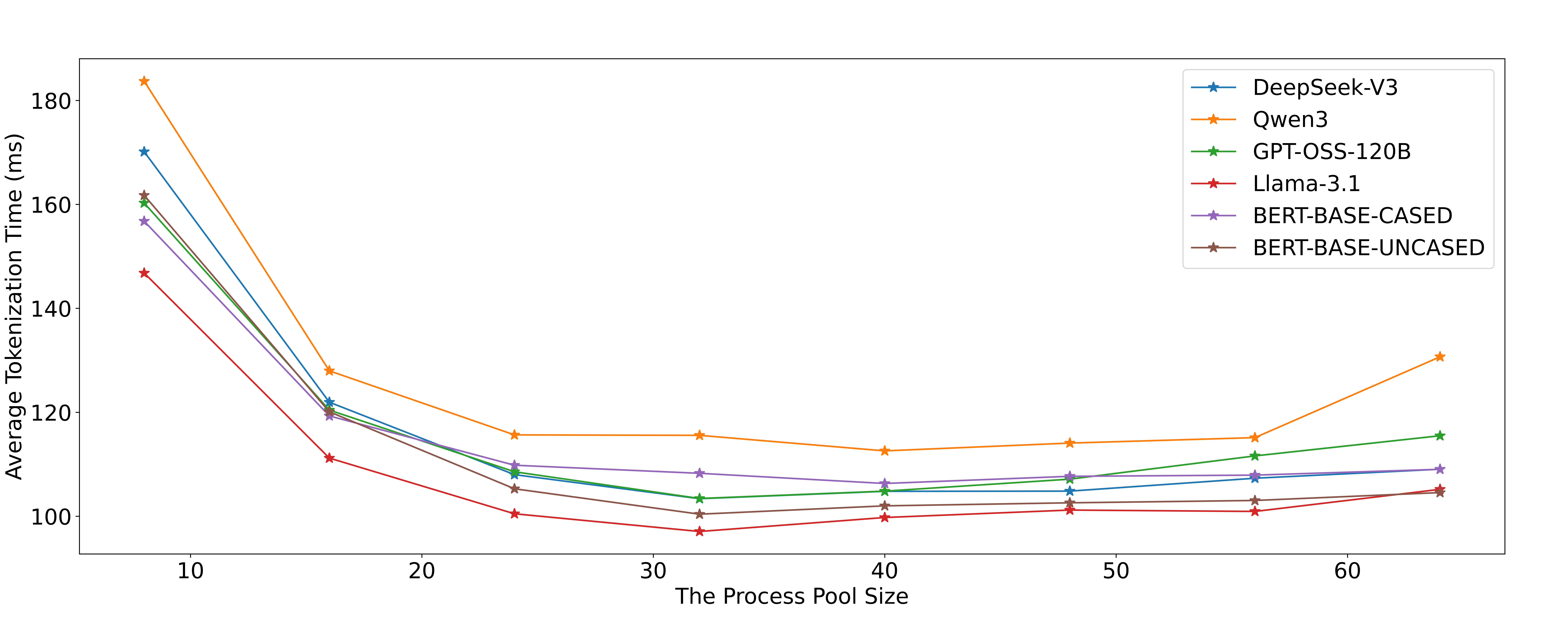}
    \caption{Performance of tokenization using our framework on the LongBenchV2 dataset with different process pool sizes. The horizontal axis represents the process pool size, and the vertical axis represents the time consumption for tokenization.}
    \label{process_num}
\end{figure}

Although the batch size in long-context inference scenarios is generally small, to provide a more comprehensive comparison between our framework and HuggingFace TokenizerFast, we conduct tokenization experiments using the Qwen3 tokenizer on the LongBenchV2 dataset. The performance of our framework and HuggingFace TokenizerFast (HF tokenizer) based on different batch sizes is obtained. Additionally, considering that the HuggingFace tokenizer can be accelerated with multi-processing at the sample level, we also implement a sample-level multi-processing version of the HuggingFace TokenizerFast (MP HF tokenizer) as a stronger baseline. The final results are shown in \Cref{batch_size}.

According to the experimental results, the average tokenization time of our framework across different batch size settings remains significantly lower than that of the HF tokenizer. When the batch size is small, the processing time of our framework is less than that of the MP HF tokenizer. As the batch size increases, the processing time of our framework begins to approach that of the MP HF tokenizer. The batch size corresponding to the intersection point between our framework and the MP HF tokenizer is influenced by the process pool size—the larger the pool size, the larger the batch size at which the intersection occurs.

\subsubsection{The Process Pool Size's Impact}

To investigate the impact of the number of processes on tokenization time, we conduct experiments using process pools of varying sizes on the LongBenchV2 dataset. The results are shown in the \Cref{process_num}. Overall, as the process pool increases, the tokenization time of our framework continuously decreases. However, a larger number of processes does not necessarily lead to better acceleration performance. Among the configurations with an increasing number of processes, there exists an optimal number of processes that minimizes the time required for tokenization.

\subsubsection{CPU Device's Impact}
\begin{figure}
    \centering
    \includegraphics[width=0.7\linewidth]{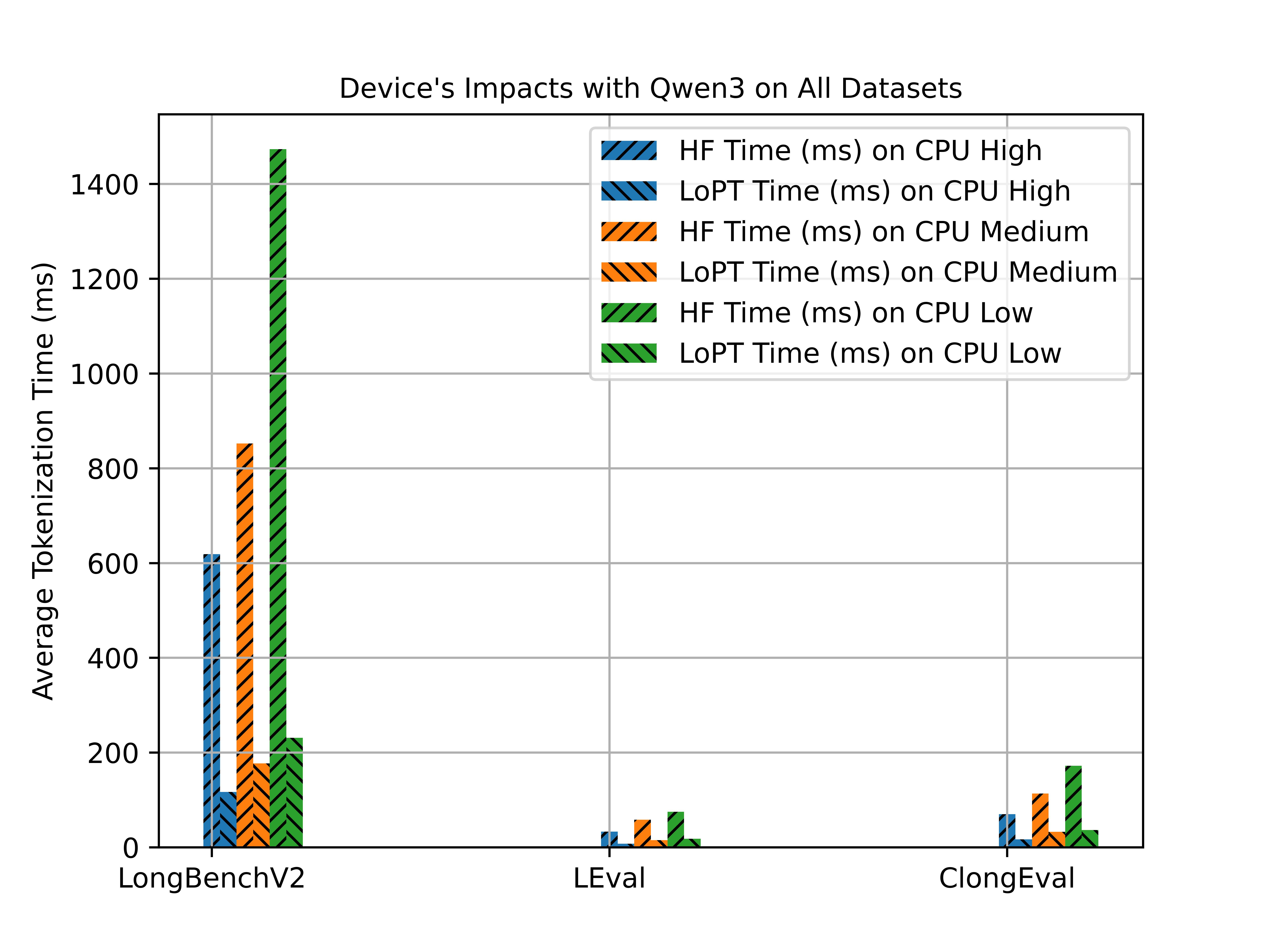}
    \caption{Performance of Qwen3 HuggingFace TokenizerFast and our framework working on devices with different computation capacity (CPU High $>$ CPU Medium $>$ CPU Low). The experiment is conducted on three datasets. The horizontal axis represents the dataset, and the vertical axis represents the time consumption for tokenization (ms). "/" represents the HuggingFace TokenizerFast, and "$\backslash$" represents our framework.}
    \label{cpu}
\end{figure}

Currently, tokenization computation is primarily performed on the CPU side. The time consumption of a tokenizer varies when working on CPUs with different computing capabilities. To investigate the performance differences of our framework on CPUs with varying computing capabilities, we use the Qwen3 tokenizer as the base tokenizer and conduct experiments on three CPU devices with different computing capabilities (CPU High $>$ CPU Medium $>$ CPU Low). The time performance of both HuggingFace and our framework was obtained. Specific results are shown in \Cref{cpu}.

We can observe that the time consumption of our framework is more consistent across different computational power levels compared to HuggingFace TokenizerFast, indicating that it is less affected by CPU capability and can achieve lower computational latency even on CPUs with lower performance. Moreover, the acceleration effect of our framework is more pronounced on CPUs with lower computational power. These results demonstrate that the performance of our framework is minimally influenced by CPU capability, exhibiting strong compatibility across various computing devices.

\section{Conclusion}

In this work, we identified tokenization as a critical yet under-optimized bottleneck in long-context LLM inference. While parallel tokenization offers a viable path to acceleration, its adoption has been hampered by output inconsistencies. To this end, we proposed LoPT, a novel lossless parallel tokenization framework that effectively resolves the inconsistency problem through character-position-based token merge and dynamic chunk length adjustment. Extensive experimental results validate that our framework achieves significant tokenization speedups across diverse datasets while producing outputs that are perfectly identical to those from standard tokenization. Furthermore, we provided a theoretical guarantee of losslessness and thorough analytical experiments to elucidate the framework's performance characteristics. Looking forward, LoPT can be seamlessly integrated into existing long-context inference systems to alleviate the tokenization bottleneck. We believe our work opens up a new direction for optimizing inference pipelines and contributes to the broader goal of efficient large-scale language model deployment. Future work will focus on further optimizing the chunking strategy and extending the framework to a wider range of tokenizers.

\bibliographystyle{ieeetr}
\bibliography{sample}

\section{Proof of \Cref{thm:lossess}}
\label{app:proof}

In this section, we prove \Cref{thm:lossess} for WordPiece-based tokenization and BPE-based tokenization. For simplicity, we denote 
$F(S)$ as the tokenization of sequence $S$, and $s(t)$ is the text span in $S$ corresponding to $t$. 

Without loss of generality, we only consider that $S$ is divided into two parts, then \Cref{thm:lossess} can be reformulated as the following lemma:

\begin{lemma}
$S = S_1 + S_o + S_2$, where $S_o$ is the overlapped part. Then 
$$
F(S_1 + S_o) = [t^{(l)}_1,t^{(l)}_2,...,t^{(l)}_m], \quad F(S_o+S_2) = [t^{(r)}_1,t^{(r)}_2,...,t^{(r)}_n].
$$
If $s(t^{(l)}_{i_1}) = s(t^{(r)}_{i_2}), s(t^{(l)}_{i_1+1}) = s(t^{(r)}_{i_2+1}),...,s(t^{(l)}_{i_1+k}) = s(t^{(r)}_{i_2+k})$, and $$\text{len}(s(t^{(l)}_{i_1})+s(t^{(l)}_{i_1+1}) + ... + s(t^{(l)}_{i_1+k})) > T,$$
where $T = \text{maxlen of words in the vocabulary}$, then we have
$$
F(S) = (t^{(l)}_1,t^{(l)}_2,...,t^{(l)}_{i_1}, t^{(l)}_{i_1+1},
...,t^{(l)}_{i_1+k},t^{(r)}_{i_2+k+1},...,t^{(r)}_{n}).
$$
\end{lemma}



\subsection{Provement for WordPiece-based Tokenization}

\begin{lemma} \label{lem:lem1}
If the first token generated during the tokenization of $S$ is t (where the text span in $S$ corresponding to t is $s(t)$) and there exists a substring $S'$ of $S$ satisfying $S'\subseteq S$ and $s(t) \subseteq S'$ then the first token generated during the tokenization of $S'$ is also $t$.
\end{lemma}

The validity of this lemma is self-evident.

\begin{lemma} \label{lem:lem2} If
    $F(S) = (t_1, t_2, ...,t_n)$, 
    then for $1\leq l < r \leq n$,  
    $$F([s_l, ...,s_r]) = (t_l, ..., t_r).$$
\end{lemma}
\begin{proof}
Considering the tokenization process of the original string $S$, we examine it token by token. If a token lies between $t_l$ and $t_r$, it will also appear in $F([t_l, ..., t_r])$. If a token does not lie between $t_l$ and $t_r$, it does not affect the result. This completes the proof of the lemma.
\end{proof}

Based on \Cref{lem:lem2}, if the first token generated during the tokenization of $S$, the sub-string of $S$ at left $t$ is $S_l$, the sub-string of $S$ at right $t$ is $S_r$, which means that $S = S_l + s_t +S_r$, then 

\begin{equation}
    F(S) = F(S_l) + F(s_t) + F(S_r)
\end{equation}

Returning to the original problem, we denote the split into left and right segments. After independent tokenization, the sub-string corresponding to overlapping tokens, i.e. $t^{(l)}_{i_1}, t^{(l)}_{i_1+1},
...,t^{(l)}_{i_1+k}$, are denoted as $M$. It is not difficult to prove that $M$ is a continuous string. Denoting the string to the left of $M$ as $L$ and the string to the right of $M$ as $R$, we have $S = L + M + R$. Further, we have the following result:

If the length of $M$ is greater than or equal to the maximum length of $W_i$, then the final result is correct, i.e.,
\begin{equation}
F(S) = F(L) + F(M) + F(R)
\end{equation}

\begin{proof}
    
We prove this by the method of infinite descent. If the theorem does not hold, there exists a counterexample with the smallest $len(S)$.
Based on the given conditions and Lemma 2, we have

$$F(L + M) = F(L) + F(M),$$
$$F(M + R) = F(M) + F(R)$$

Note that $S = L + M + R$. Consider the first token $t$ generated during the tokenization of $S$. We proceed with a case-by-case analysis:

\paragraph{Case 1 $s(t) \subseteq L$ or $s(t) \subseteq R$:}
We may assume that $s(t) \subseteq L$. Denote $L = Ll + t + Lr$. Then $t$ is also the first token of $L$ during the tokenization process. Thus,
$$F(L) = F(Ll + s(t) + Lr) = F(Ll) + t + F(Lr),$$
$$ F(S) = F(Ll) + t + F(Lr + M + R),$$
$$ F(S) \neq F(L) + F(M) + F(R). $$
Thus we have 
$$ F(Ll) + t + F(Lr+M+R) \neq F(Ll) + t + F(Lr) + F(M) + F(R), $$
$$ F(Lr+M+R) \neq F(Lr) + F(M) + F(R). $$
Considering the combination $Lr + M + R$, it forms a new and shorter counterexample, leading to a contradiction.

\paragraph{Case 2 $s(t) \subseteq M$:}
In this case, $t$ is the first token in $M$. Thus, 
$$M = Ml + s(t) + Mr,$$
$$F(S) = F(L+Ml) + t + F(Mr+R),$$
$$F(L+M) = F(L+Ml) + t + F(Mr),$$
$$F(M+R) = F(Ml) + t + F(Mr+R),$$
$$F(M) = F(Ml) + t + F(Mr).$$
On the other hand, $$F(L+M) = F(L)+F(M) = F(L) + F(Ml) + t + F(Mr),$$
$$F(M + R) = F(M)+F(R) = F(Ml) + t + F(Mr) + F(R),$$
thus 
$$ F(L + Ml) = F(L) + F(Ml), \quad F(Mr + R) = F(Mr) + F(R).
$$
Finally, we have
$$F(S) = F(L) + F(M) +F(R).$$
\paragraph{Case 3 $s(t) \subseteq L+M$ or $s(t) \subseteq M+R$, but case 1 and 2 are not satisfied:}
Without loss of generality, assume $s(t)\subseteq L+M$. Since $t$ is the first token generated in $S$, it is also the first token generated in $L+M$. However, $s(t)$ does not belong to $L$ or $M$, meaning $t$ does not appear in $F(L)$ nor in $F(M)$, yet $t$ is present in $F(L+M)$. This contradicts

$$F(L+M)=F(L)+F(M)$$
\end{proof}

\paragraph{Case 4 $s(t)$ does not belong to $L+M$ nor to $M+R$:}
The starting point of $t$ is before the starting point of $M$, and the ending point of $t$ is after the ending point of $M$. Thus, $len(t) > len(M)$, which contradicts $len(M) \geq \max\ len$.

\subsection{Proof for BPE-based Tokenization}

Assume there exists a counterexample with the minimal number of fragments, where fragments refer to the unmerged parts. Then, by considering the first merge operation, the remaining steps of the proof are consistent with the proof for WordPiece.
\end{document}